\documentclass[letterpaper]{article} % DO NOT CHANGE THIS
\usepackage{aaai25}  % DO NOT CHANGE THIS
\usepackage{times}  % DO NOT CHANGE THIS
\usepackage{helvet}  % DO NOT CHANGE THIS
\usepackage{courier}  % DO NOT CHANGE THIS
\usepackage[hyphens]{url}  % DO NOT CHANGE THIS
\usepackage{graphicx} % DO NOT CHANGE THIS
\usepackage{natbib}  % DO NOT CHANGE THIS AND DO NOT ADD ANY OPTIONS TO IT
\usepackage{caption} % DO NOT CHANGE THIS AND DO NOT ADD ANY OPTIONS TO IT
\usepackage{algorithm}
\usepackage{algorithmic}
\usepackage{amsfonts}
\usepackage{booktabs}
\usepackage{multirow}
\usepackage{makecell}
\usepackage{color}
\usepackage{amsthm}

\newtheorem{theorem}{Theorem} 
\newtheorem{definition}{Definition}
\newtheorem{assumption}{Assumption}

\usepackage{newfloat}
\usepackage{listings}
\DeclareCaptionStyle{ruled}{labelfont=normalfont,labelsep=colon,strut=off} % DO NOT CHANGE THIS
\floatstyle{ruled}
\newfloat{listing}{tb}{lst}{}
\floatname{listing}{Listing}
\title{Towards Learnable Anchor for Deep Multi-View Clustering}
\author{
    Bocheng Wang\textsuperscript{\rm 1}, Chusheng Zeng\textsuperscript{\rm 1}, Mulin Chen\textsuperscript{\rm 1}\thanks{Corresponding author.}, Xuelong Li\textsuperscript{\rm 2} \\
}
\affiliations{
    \textsuperscript{\rm 1}School of Artificial Intelligence, OPtics and
    ElectroNics (iOPEN),  \\
    Northwestern Polytechnical University, China \\
    
    \textsuperscript{\rm 2}Institute of Artificial Intelligence (TeleAI), China Telecom, China \\
    wangbocheng@mail.nwpu.edu.cn, zcs@mail.nwpu.edu.cn, chenmulin@nwpu.edu.cn, xuelong\_li@ieee.org
}

\usepackage{bibentry}
\begin{document}

\maketitle

\begin{abstract}
Deep multi-view clustering incorporating graph learning has presented tremendous potential. Most methods encounter costly square time consumption w.r.t. data size. Theoretically, anchor-based graph learning can alleviate this limitation, but related deep models mainly rely on manual discretization approaches to select anchors, which indicates that 1) the anchors are fixed during model training and 2) they may deviate from the true cluster distribution. Consequently, the unreliable anchors may corrupt clustering results. In this paper, we propose the Deep Multi-view Anchor Clustering (DMAC) model that performs clustering in linear time. Concretely, the initial anchors are intervened by the positive-incentive noise sampled from Gaussian distribution, such that they can be optimized with a newly designed anchor learning loss, which promotes a clear relationship between samples and anchors. Afterwards, anchor graph convolution is devised to model the cluster structure formed by the anchors, and the mutual information maximization loss is built to provide cross-view clustering guidance. In this way, the learned anchors can better represent clusters. With the optimal anchors, the full sample graph is calculated to derive a discriminative embedding for clustering. Extensive experiments on several datasets demonstrate the superior performance and efficiency of DMAC compared to state-of-the-art competitors.
\end{abstract}

\section{Introduction}

%, such as graph-based MVC \cite{GMC, dornaika2024towards}, matrix decomposition based MVC \cite{zhang2021multi, wang2023multi}, and deep MVC \cite{CONAN, SURER}.

In real-world data acquisition, a sample is often recorded from different views or sources \cite{cui2024novel}, constituting multi-view data. For instance, an image can be distilled by multiple feature descriptors (e.g., color, shape, and spatial relation), and a document can be interpreted from multiple perspectives (e.g., topic distribution, word sequence, and word frequency). To adapt to the actual application environment, many types of unsupervised Multi-View Clustering (MVC) have emerged. Among them, deep MVC incorporates the enormous advantages of neural networks on representation learning, and has presented outstanding performance in real-world scenarios. 

In general, deep MVC sets up a separate auto-encoder for each view, and then the learned multiple view-specific embeddings are fused into a consensus to infer clusters. To mitigate the inter-view conflict and perceive multi-view consistency, many representation alignment strategies are leveraged, such as contrastive learning \cite{MFLVC, hu2023joint}, and label distribution alignment \cite{MAGCN, SDMVC, liu2024deep}. Overall, most deep MVC models emphasize the discriminability of the output embedding to improve clustering. 

Recently, some works integrate the graph learning theory into deep MVC, which considers the structure relationship between samples while learning the discriminative embedding \cite{DealMVC, GCFAgg, SURER}. These novel models capture the topological structure by constructing a data similarity graph, and thus optimizing representation learning based on graph data mining techniques, such as graph convolution network and structure preservation scheme \cite{wang2023multi, DFP-GNN}. These methods usually require computing the edge weights between any two samples to build a full sample graph, which leads to costly time complexity $O(n^2)$ where $n$ is the amount of samples. One theoretical solution to this problem is anchor graph learning that can accelerate the training to linear time w.r.t. $n$ \cite{CTCC, DMCAG}. The performance of anchor graph learning and final clustering heavily depends on the anchor quality. However, most deep models employ non-differentiable manual procedures (e.g., random selection and $k$-means) to determine the final anchors, which means the anchors are not learnable. Ideally, each sample should be represented by a certain anchor, and the anchors should reflect the cluster distribution.  If the selected anchors cannot represent the original samples and cluster centroids well, the clustering result may be adversely affected.

\begin{figure*}
	\center
	\includegraphics[width=0.98\textwidth]{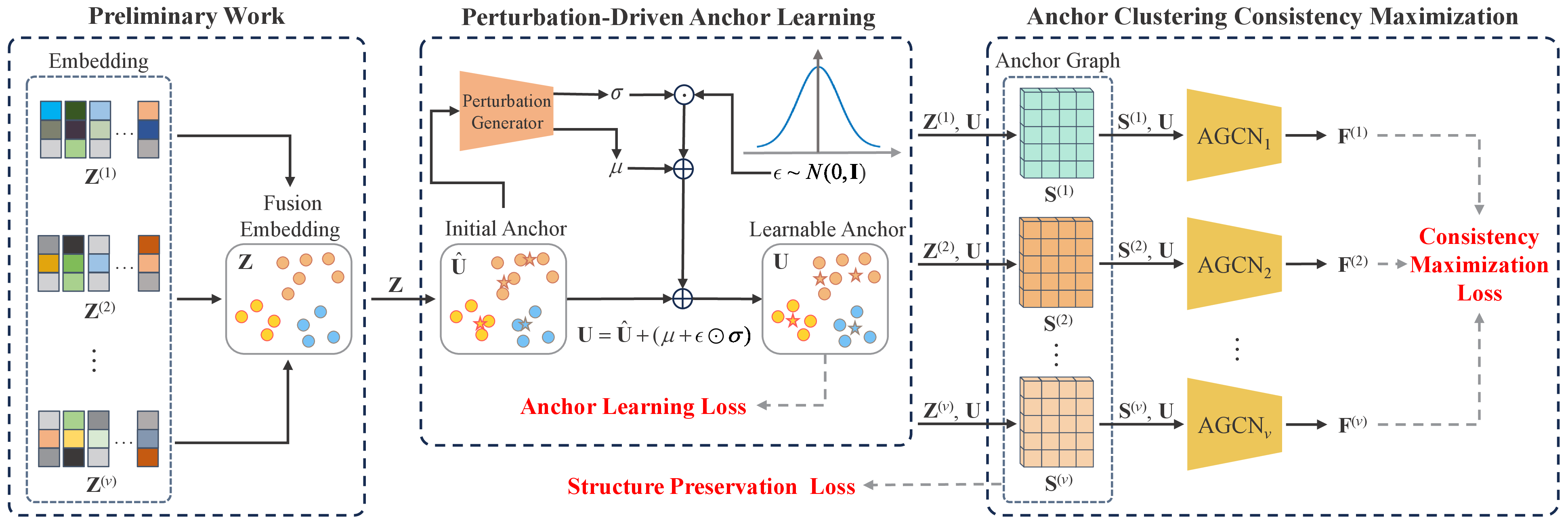} 
	\caption{Pipeline of DMAC. Note that the encoders are omitted. For the $a$-th view, $\mathbf{Z}^{(a)}$ is the data embedding, $\mathbf{A}^{(a)}$ is the anchor graph, $\mathrm{AGCN}_a$ is the corresponding anchor graph convolution network, and $\mathbf{F}^{(a)}$ is the anchor clustering distribution that records the probability of an anchor belonging to each cluster. $\mathbf{Z}$ is the shared fusion embedding among views. $\mathbf{U}$ represents the learnable anchors injected with the perturbation. The overall framework is updated by minimizing Eq. (\ref{eq:overall_loss}). The final result is gained by performing $k$-means on the convergent $\mathbf{Z}$.}
	\label{fig:flowchart}
\end{figure*}

To remedy the problems, we establish the Deep Multi-view Anchor Clustering (DMAC) model. 
As shown in Fig. \ref{fig:flowchart}, the pipeline of DMAC can be partitioned into three main items. Firstly, the positive-incentive perturbation is injected into the initial anchors to yield a learnable mechanism for the anchors.
Then, anchor graph convolution is used to produce the cluster indicator of anchors for each view, and the cross-view agreement is obtained by the proposed mutual information maximization loss. Finally, the learned anchors are leveraged to reveal the structural relationship between samples to impel a discriminative embedding. The key contributions of this paper are listed as follows.

\begin{itemize}
	\item Anchor graph learning is incorporated into the deep multi-view clustering framework. Compared to most competitors, the proposed model can capture structural information with learnable anchors in a high-efficiency linear time to ameliorate embedding learning.
	
	\item A perturbation-driven mechanism is proposed to improve the anchor quality adaptively during model training. By infusing beneficial perturbations, the learned anchors are adjusted to foster a distinct similarity relationship with the samples, thus representing the original data more comprehensively.
	
	\item An anchor graph convolution network is devised to infer the anchor clustering distribution of each view. The mutual information among multiple distributions is maximized to explore the consistent anchor clustering distribution, so as to produce clustering-oriented anchors. 
\end{itemize}

\section{Related Work}

\subsection{Deep Multi-View Clustering}

Benefiting from the powerful representation learning ability of neural networks, deep MVC has achieved dominant clustering performance in practical applications \cite{DMVC_survey2}. Generally, deep MVC leverages auto-encoders to learn a series of view-specific embeddings, and then executes representation fusion to infer the consensus cluster indications. Many models are proposed to promote the clustering-friendly deep embedding. DAMC \cite{DAMC} introduces the adversarial training mechanism to improve the discriminability of the learned embedding. DEMVC \cite{DEMVC} pursues a consistent cluster structure between views to exploit the cross-view complementary information. In \cite{alignment}, the importance of representation alignment for MVC is analyzed theoretically, and a contrastive learning based deep MVC method is established. MFLVC \cite{MFLVC} performs the feature- and cluster-level contrastive learning simultaneously to inhibit the adverse effects of view-private information. CVCL \cite{CVCL} advocates unifying the cluster assignment by multi-view contrastive learning. 

The abovementioned representatives have presented good clustering capacity in experiments. However, most models focus on the intrinsic features while neglecting the structure relationship among samples that is essential to detect the clustering distribution \cite{MvSCN, structure_importance, DCGL}. Recently, some scholars introduce the graph learning theory into deep MVC to mine the structural information explicitly. CMGEC \cite{CMGEC} employs a graph fusion network to integrate multiple structural graphs into a consensus. DFP-GNN \cite{DFP-GNN} parallels graph neural networks with auto-encoders to learn the feature- and structure-level embedding simultaneously. DealMVC \cite{DealMVC} constructs the data similarity graph to guide clustering-oriented contrastive learning. GCFAggMVC \cite{GCFAgg} utilizes the similarity relationship among embeddings to ameliorate representation fusion. SURER \cite{SURER} concatenates multiple view-specific graphs into a heterogeneous graph to explore the complementary relationship among views via the heterogeneous graph neural network. Those deep MVC models that incorporate the graph structure information have presented enormous development potential. Nevertheless, existing methods usually require computing the full sample graph, which encounters a squared complexity $O(n^2)$ where $n$ is the number of samples. In this paper, we plan to introduce the anchor theory to achieve deep MVC in linear time.

\subsection{Anchor-Based Multi-View Clustering}

Anchors, also known as landmarks, are widely used in graph-based and sub-space MVC \cite{ anchor_widely2}. An anchor is the representative of local data \cite{landmark}. By learning an anchor graph that records the adjacency relation between samples and anchors, the similarity among samples can be estimated approximately to derive the clusters. SFMC \cite{SFMC} fuses multiple anchor graphs into a consensus bipartite graph with the rank constraint. MSGL \cite{MSGL} uses the data self-expression property to realize the adaptive anchor selection and graph optimization. FDAGF \cite{FDAGF} allows multiple anchor combinations as inputs to improve the flexibility and generalization ability. E$^\mathrm{2}$OMVC \cite{EOMVC} calculates the spectral embeddings of anchor graphs and fuses them into the final cluster representation. CAMVC \cite{CAMVC} utilizes the estimated labels to optimize cluster-wise anchors. 

Since the number of anchors is much smaller than the sample size, anchor-based MVC has outstanding efficiency. Nevertheless, most relevant algorithms are limited by shallow graph learning that the anchor graph is calculated via the original features directly. There is little exploration \cite{CTCC, DMCAG} of anchor-based deep MVC, which neglects the optimization of anchor quality with model training. Inspired by the positive-incentive noise \cite{PNoise}, we plan to generate positive noise perturbation to guide high-quality anchor learning, and design the anchor graph convolution module to capture cross-view anchor clustering consistency.

\section{Methodology}

In this section, the proposed DMAC is elaborated. DMAC generates positive perturbation to ameliorate the anchors shared among views, and utilizes anchor graph convolution to extract the cluster distribution of anchors.

\textbf{Notations:} matrices and vectors are expressed as uppercase and lowercase letters, respectively. For a matrix $\mathbf{X}$, both $\mathbf{X}_{i}$ and $\mathbf{x}_{i}$ mean the $i$-th row. $||\mathbf{x}_{i}||_{1}$, $||\mathbf{x}_{i}||_{2}$, and $||\mathbf{X}||_{\mathrm{F}}$ denote $\ell_{1}$, $\ell_{2}$, and Frobenius norm, respectively. 

\subsection{Preliminary Work}

Denote \{$\mathbf{X}^{(1)}$, $\mathbf{X}^{(2)}$, $\cdots$, $\mathbf{X}^{(v)}$\} as the multi-view data with $n$ samples, $v$ views, and $c$ clusters. DMAC follows the mainstream deep MVC framework consisting of view-specific embedding learning and embedding fusion. In this part, the framework is introduced briefly to pave the subsequent innovative modules.

Specifically, the unshared encoder is used to learn view-specific deep embedding
\begin{equation}
	\begin{array}{c}
		\mathbf{Z}^{(a)}=\mathrm{Encoder}_a(\mathbf{X}^{(a)}). 
	\end{array}
\end{equation}

Based on the resultant embeddings \{$\mathbf{Z}^{(1)}$, $\mathbf{Z}^{(2)}$, $\cdots$, $\mathbf{Z}^{(v)}$\}, the simple but effective average weighting \cite{SURER} is employed to calculate the fusion embedding
\begin{equation}
	\label{eq:fuse}
	\begin{array}{c}
		\mathbf{Z} = \frac{1}{v} \sum \limits_{i}^v \mathbf{Z}^{(i)},
	\end{array}
\end{equation}
which is fed into a clustering algorithm to gain the labels.

Very lately, graph learning theory is introduced into the above framework, aiming to extract the topological structure of sample space to enhance embedding learning. Most related models encounter an expensive time complexity $O(n^2)$. Differently, we integrate anchor graph learning into deep MVC to reduce the complexity to linear time.

\subsection{Perturbation-Driven Anchor Learning}

Anchor graph learning requires estimating multiple anchors in advance. Existing models usually use the manual setting strategy to select anchors, which inhibits the learnability of anchors. Therefore, we construct a generator to learn the perturbation to adjust anchors, and design the anchor learning loss to obtain the positive perturbation that improves anchor quality.

\subsubsection{Perturbation Generation Network.} Denoting $\mathbf{\widehat{U}}$ as the $m$ initial shared anchors obtained by performing $k$-means on $\mathbf{Z}$, the generation network produces learnable perturbation $\varepsilon$ to inject $\mathbf{\widehat{U}}$. In this way, the anchors can be optimized by updating $\varepsilon$ through backpropagation.

To begin with, the initial perturbation $\epsilon$ is sampled from the standard multivariate Gaussian distribution
\begin{equation}
	\begin{array}{c}
		\epsilon \sim N(0, \mathbf{I}),
	\end{array}
\end{equation}
where $\mathbf{I}$ is an identity matrix, and $\epsilon$ has the same dimensionality as $\mathbf{\widehat{U}}$.

Then, we use a pseudo-siamese perceptron to simulate the mean $\mu$ and deviation $\sigma$ of the perturbation, which are formulated as
\begin{equation}
	\begin{array}{c}
		\mu = \mathrm{MLP}_{\mu}(\mathbf{\widehat{U}}), ~
		\sigma = \mathrm{MLP}_{\sigma}(\mathbf{\widehat{U}}).
	\end{array}
\end{equation}
Consequently, the perturbation $\varepsilon$ is updated as
\begin{equation}
	\label{eq:noise}
	\begin{array}{c}
		\varepsilon = \mu + \sigma \odot \epsilon,
	\end{array}
\end{equation}
where $\odot$ refers to the Hadamard product. Eq. (\ref{eq:noise}) meets the reparameterization trick \cite{trick} that optimizes $\varepsilon$ by backpropagation.
Finally, the anchor matrix is updated as 
\begin{equation}
	\begin{array}{c}
		\mathbf{U} = \mathbf{\widehat{U}} + \varepsilon.
	\end{array}
\end{equation}

Since $\varepsilon$ is optimized gradually during training, the learnability of the anchor matrix $\mathbf{U}$ is achieved.

\subsubsection{Anchor Learning Loss.} Given the learnable anchors, the relationship between samples and anchors is explored to improve the quality of anchors. 

Intuitively, the ideal anchors can be considered as sub-centroids, which are appropriately dispersed, and each sample is strongly associated with a corresponding anchor. Therefore, the similarity of the sample and anchors is crucial to evaluate the anchor quality. Denoting $q_{i}^{(a)} \in \mathbb{R}^{1\times m}$ as the similarity vector between the sample $\mathbf{Z}_i^{(a)}$ and $m$ anchors $\mathbf{U}$, $q_{ij}^{(a)}$ is computed as
\begin{equation}
	\label{eq:similarity}
	\begin{array}{c}
		q_{ij}^{(a)} = 
		\frac{\left(1+||\mathbf{Z}_i^{(a)} - \mathbf{U}_j||_2^2\right)^{-1}}{\sum \limits_{k}^m \left(1+||\mathbf{Z}_i^{(a)} - \mathbf{U}_k||_2^2\right)^{-1}}.
	\end{array}
\end{equation}

Then, we introduce the positive-incentive noise theory to pave the anchor learning loss. In other words, the perturbation $\varepsilon$ is regarded as the potential positive noise, which satisfies the definition in \cite{PNoise}.
\begin{definition}
	\label{def:pi}
	Mathematically, the positive noise $\epsilon_\pi$ satisfies
	\begin{equation}
		\label{eq:pinoise}
		\begin{array}{c}
			E(\mathcal{T} | \epsilon_\pi) < E(\mathcal{T}),
		\end{array}
	\end{equation}
	where $\mathcal{T}$ represents a specific downstream task, and $E(\cdot)$ computes the information entropy. 
\end{definition}

According to Definition \ref{def:pi}, the positive-incentive noise aims to reduce the uncertainty of a specific downstream task. Considering that the anchors are representatives of the original data, the uncertainty of anchor selection task is mainly reflected by the distribution of $q_i^{(a)}$. To be specific, for a sample $\mathbf{Z}_i^{(a)}$, if all values in  $q_i^{(a)}$ are very close, the relationship between the sample and anchors is ambiguous, which indicates a high uncertainty. That is to say, we need to learn a extremely unbalanced $q_i^{(a)}$, wherein one element is much larger than the others.

Therefore, for the $a$-th view, the task entropy $E(\mathcal{T} | \epsilon_\pi)$ in Eq. (\ref{eq:pinoise}) can be quantified as 
\begin{equation}
	\label{eq:anchor_learning}
	\begin{array}{c}
		{{\cal L}_{AL}^{(a)}} = \frac{1}{n} \sum \limits_{i}^{n} E\left(q_i^{(a)}\right) = -\frac{1}{n} \sum \limits_{i}^{n} \sum \limits_{j}^{m} q_{ij}^{(a)} \log \left(q_{ij}^{(a)}\right).
	\end{array}
\end{equation}

By minimizing Eq. (\ref{eq:anchor_learning}), each sample tends to hold a highly correlated anchor, such that the anchors can represent the data distribution well. In the following part, we aim to make the anchors aligned with the cluster distribution.

\subsection{Anchor Clustering Consistency Maximization}

To learn clustering-oriented anchors, we devise anchor graph convolution to infer the anchor clustering distribution of each view, and introduce mutual information to capture cross-view anchor clustering consistency, so as to provide training guidance for anchor learning.

\subsubsection{Anchor Graph Learning.} Since the proposed anchor graph convolution network requires graph data as input, we first present anchor graph learning.

Anchor graph records the structural dependence between samples and anchors. For the $a$-th view, an ideal anchor graph $\mathbf{S}^{(a)}$ respects the following assumption.
\begin{assumption}
	\label{ass:agl}
	The edge $s_{ij}^{(a)}$ is negatively correlated with the distance between the corresponding nodes $\mathbf{Z}^{(a)}_i$ and ${{\mathbf{U}}_j}$.
\end{assumption}
Assumption \ref{ass:agl} also reflects the fundamental clustering scenario, that points with small distances are more likely to be within the same cluster. Hence, for the $a$-th view, the anchor graph learning problem can be expressed as
\begin{equation}
	\label{eq:agl}
	\begin{array}{c}
		\mathop {\min }\limits_{\mathbf{S}^{(a)}} {\rm{ }}\mathop \sum \limits_{i}^n \sum \limits_{j}^m {\rm{||}}{\mathbf{Z}_i^{(a)}} - {{\mathbf{U}}_j}{\rm{||}}_2^2{s_{ij}^{(a)}} + \gamma {\rm{||}}\mathbf{S}^{(a)}{\rm{||}}_{\rm{F}}^2,\\
		s.t. \forall i ~ {\rm{ }}{\rm{ ||}}{\mathbf{s}_i^{(a)}}{\rm{|}}{{\rm{|}}_1} = 1,{\rm{ }}0 \le {\mathbf{s}_i^{(a)}} \le 1,
	\end{array}
\end{equation}
where $\mathbf{s}_i^{(a)}$ is the $i$-th row of the anchor graph $\mathbf{S}^{(a)} \in \mathbb{R}^{n\times m}$, and the second term evades that each sample only connects with the nearest anchor. The constraint that the sum of each row in $\mathbf{S}^{(a)}$ is 1 aims to approach Assumption \ref{ass:agl}. 

According to \cite{ULGE}, problem (\ref{eq:agl}) can be solved with an efficient closed-form solution, which also evades the selection of parameter $\gamma$.

\subsubsection{Anchor Graph Convolution.} Based on the view-specific anchor graph $\mathbf{S}^{(a)}$ and shared anchors $\mathbf{U}$, we introduce graph convolution \cite{GCN} to calculate the anchor clustering distribution, leading to Anchor Graph Convolution Network (AGCN). 

For the $a$-th view, the row in anchor clustering distribution $\mathbf{F}^{(a)} \in \mathbb{R}^{m\times c}$ records the probability of an anchor belonging to each cluster. Specifically, the forward propagation of $l$-th hidden layer in $\mathrm{AGCN_a}$ is
\begin{equation}
	\label{eq:agcn}
	\begin{array}{c}
		{\mathbf{F}^{{(a)}(l + 1)}} = \varphi \left(\mathbf{D}_\mathbf{S^{\mathnormal{(a)}}}^{-1} \left(\mathbf{S}^{(a)}\right)^{\mathrm{T}}\mathbf{S}^{(a)}\mathbf{F}^{{(a)}(l)}
		\mathbf{W}^{(l)}\right),
	\end{array}
\end{equation}
where ${\mathbf{D}_{\mathbf{S}^{(a)}}} \in \mathbb{R}^{m\times m}$ is the diagonal degree matrix of $\mathbf{S}^{(a)}$ that the $j$-th item is $\sum _i^n s_{ij}^{(a)}$, $\left( \mathbf{D}_\mathbf{S^{\mathnormal{(a)}}}^{-1} \left(\mathbf{S}^{(a)}\right)^{\mathrm{T}}\mathbf{S}^{(a)}\right) \in \mathbb{R}^{m\times m}$ is the symmetric and doubly stochastic anchor similarity graph \cite{zhang2022non} that conforms to the criterion of GCN, $\mathbf{W}^{(l)}$ is the parameter matrix, and $\varphi(\cdot)$ is a certain activation function. Note that $\mathbf{F}^{{(a)}(0)} = {{\mathbf{U}}}$. 

The neuron number of the last layer in AGCN is the cluster number $c$, and the softmax function is used for activation. Multiple AGCNs do not share parameters to learn the anchor cluster structure for each view.

\subsubsection{Consistency Maximization Loss.}

We adopt Mutual Information (MI) to measure the difference among \{$\mathbf{F}^{(1)}$, $\mathbf{F}^{(2)}$, $\cdots$, $\mathbf{F}^{(v)}$\}. Compared to the widespread KL divergence, MI satisfies symmetry to increase computational efficiency. The large MI, the more similar the two distributions is. Based on MI, the consistency maximization loss for the $a$-th view is
\begin{equation}
	\label{eq:cm_loss}
	\begin{array}{c}
		{{\cal L}_{CM}^{(a)}} = - \frac{1}{m} \sum \limits_{b=a+1}^v \sum \limits_{i}^m MI \left(\mathbf{F}^{(a)}_i, \mathbf{F}^{(b)}_i \right),
	\end{array}
\end{equation}
where
\begin{equation}
	\begin{array}{c}
		MI \left(\mathbf{F}^{(a)}_i, \mathbf{F}^{(b)}_i \right) = \sum \limits_{x \sim \mathbf{F}^{(a)}_i} \sum \limits_{y \sim \mathbf{F}^{(b)}_i} p\left(x,y\right) log\left(\frac{p\left(x,y\right)}{p(x)p(y)}\right).
	\end{array}
\end{equation}

By minimizing Eq. (\ref{eq:cm_loss}), multiple anchor clustering distributions are aligned to relieve the adverse effects of view conflict and view-private information. Unlike existing deep MVC, the module achieves cross-view consensus from the perspective of anchors rather than samples, which is conducive to propelling an explicit cluster structure of anchors.

\subsection{Structure Preservation via Anchor Graph}

In this part, the structural graph of samples is calculated via the learned anchor graph. On this basis, the structure preservation loss is developed to promote a discriminative multi-view embedding.  

Based on the anchor graph $\mathbf{S}^{(a)}$, the full sample graph of the $a$-view can be measured with
\begin{equation}
	\label{eq:full_graph}
	\begin{array}{c}
		\mathbf{G}^{(a)} = \mathbf{S}^{(a)} {\mathbf{D}^{{-1}}_{\mathbf{S}^{(a)}}}
		\left(\mathbf{S}^{(a)}\right)^{\mathrm{T}}.
	\end{array}
\end{equation}
Obviously, the resultant $\mathbf{G}^{(a)} \in \mathbb{R}^{n\times n}$ is a symmetric and doubly stochastic graph \cite{FDAGF}.

With the $a$-th full sample graph $\mathbf{G}^{(a)}$, the structure preservation loss is
\begin{equation}
	\label{eq:sp_loss}
	\begin{array}{c}
		{{\cal L}_{SP}^{(a)}} = \sum \limits_{i,j}^n {\rm{||}}{\mathbf{Z}_i} - {\mathbf{Z}_j}{\rm{||}}_\mathrm{2}^\mathrm{2}{g_{ij}^{(a)}},
	\end{array}
\end{equation}
which can be replaced with Eq. (\ref{eq:sp_loss_speed}) to accelerate the matrix multiplication. There is a concern that Eq. (\ref{eq:sp_loss}) may trivially cause all samples to be mapped to the same embedding, which is called representation collapse. In Theorem \ref{the:trivial}, we deduce that the anchor learning loss can be seen as a regularization term to penalize the trivial solution.

By minimizing Eq. (\ref{eq:sp_loss}), the learned fusion embedding $\mathbf{Z}$ is actuated to reserve the internal data structure of each view, that is, the samples in the same class remain compact. The complementary structural information across views is mined to ameliorate a discriminative fusion embedding for clustering performance improvement.

\subsection{Joint Loss and Optimizer}

%\begin{algorithm}[tb]
%	\caption{\textbf{DMAC}}
%	\label{alg:algorithm}
%	\textbf{\textbf{Input}}: Multi-view data \{$\mathbf{X}^{(1)}$, $\mathbf{X}^{(2)}$, $\cdots$, $\mathbf{X}^{(v)}$\}, cluster number $c$, anchor number $m$, maximum epochs $iter$, trade-off parameters $\alpha$ and $\beta$. \\
%	\textbf{\textbf{Output}}: Cluster labels. 
%	\begin{algorithmic}[1]
	%		\STATE Let $i=1$.
	%		\WHILE{$i \leq iter$}
	%		\STATE Obtain \{$\mathbf{Z}^{(1)}$, $\mathbf{Z}^{(2)}$, $\cdots$, $\mathbf{Z}^{(v)}$\} by encoders.
	%		\STATE Compute fusion embedding $\mathbf{Z}$ with Eq. (\ref{eq:fuse}).
	%		\STATE Compute initial anchor $\widehat{\mathbf{U}}$ by perform $k$-means on $\mathbf{Z}$.
	%		\STATE Obtain $\varepsilon$ by noise generation network.
	%		\STATE Adjust consensus anchor $\widehat{\mathbf{U}}$ to $\mathbf{U} = \widehat{\mathbf{U}} + \varepsilon$.
	%		\STATE Compute \{$\mathbf{S}^{(1)}$, $\mathbf{S}^{(2)}$, $\cdots$, $\mathbf{S}^{(v)}$\} with Eq. (\ref{eq:agl}).
	%		\STATE Obtain \{$\mathbf{F}^{(1)}$, $\mathbf{F}^{(2)}$, $\cdots$, $\mathbf{F}^{(v)}$\} by anchor graph convolution shown in Eq. (\ref{eq:agcn}).
	%		\STATE For each $a$, compute $ {\cal L}_{AL}^{(a)}$, ${\cal L}_{SP}^{(a)}$, and ${\cal L}_{CM}^{(a)}$.
	%		\STATE Update model by optimizing Eq. (\ref{eq:overall_loss}) with RMSprop.
	%		\STATE Let $i =i +1$.
	%		\ENDWHILE
	%		\STATE Perform $k$-means on $\mathbf{Z}$.
	%		\STATE \textbf{return} Clustering result.
	%	\end{algorithmic}
%\end{algorithm}

Combining Eqs. (\ref{eq:anchor_learning}), (\ref{eq:cm_loss}), and (\ref{eq:sp_loss}), the joint loss is
\begin{equation}
	\label{eq:overall_loss}
	\begin{array}{ll}
		{\cal L} 
		=\sum \limits_{a}^v \left({{\cal L}_{AL}^{(a)}} + \alpha {{\cal L}_{CM}^{(a)}} + \beta {{\cal L}_{SP}^{(a)}} \right) ,
	\end{array}
\end{equation}
where both $\alpha$ and $\beta$ are the trade-off parameters. 

The classical RMSprop optimizer \cite{RMSprop} is adopted to train DMAC. The final result is obtained by performing $k$-means \cite{kmeans} on the fusion representation $\mathbf{Z}$.

\section{Discussion and Analysis}

\subsection{Theoretical Advantage of Anchor Learning Loss}

In this part, we discuss that the proposed anchor learning loss shown in Eq. (\ref{eq:anchor_learning}) is beneficial for relieving representation collapse. The anchor learning loss can be regarded as a regularization term of the structure preservation loss shown in Eq. (\ref{eq:sp_loss}) to boost a discriminative fusion embedding $\mathbf{Z}$.

\begin{theorem}
	\label{the:trivial}
	Minimizing Eq. (\ref{eq:anchor_learning}) is equivalent to penalizing the trivial solution (i.e., representation collapse) to Eq. (\ref{eq:sp_loss}).
\end{theorem}
\begin{proof}
	Without loss of generality, we develop the proof from the perspective of the $a$-th view. The converse-negative proposition corresponding to the theorem is
	\begin{equation}
		\label{eq:th1}
		\begin{array}{c}
			\left(\mathbf{Z} \rightarrow \mathbf{Z^{*}} \right) \Rightarrow \left({{\cal L}_{AL}^{(a)}} \not\rightarrow 0 \right),
		\end{array}
	\end{equation}
	where $\mathbf{Z^{*}}$ is the trivial solution to Eq. (\ref{eq:sp_loss}), that is, all rows in $\mathbf{Z}$ tend to be the same.
	
	Because $\mathbf{U} \sim \mathbf{Z}$ (i.e., $\mathbf{U}$ is sampled from $\mathbf{Z}$), all anchors also tend to be the same when $\mathbf{Z} = \mathbf{Z^{*}}$. Hence, the probability distribution $q_i^{(a)}$ between sample $\mathbf{Z}_i^{(a)}$ and anchor matrix $\mathbf{U}$ is very smooth, and then the anchor learning loss reaches the upper bound. The above deduction can be formulized as
	\begin{equation}
		\begin{array}{c}
			\left(\mathbf{Z} \rightarrow \mathbf{Z^{*}} \right) \Rightarrow \left(\mathbf{U} \rightarrow \mathbf{U^{*}} \right) \Rightarrow 
			\left(\forall i ~ \forall j ~ q_{ij}^{(a)} \rightarrow \frac{1}{m}\right) \\
			\Rightarrow \left({{\cal L}_{AL}^{(a)}} \rightarrow log(m) \not\rightarrow 0\right).
		\end{array}
	\end{equation}
	
	The original proposition and converse-negative proposition possess the same truth and falsehood property. Proposition (\ref{eq:th1}) is true, so the theorem is proven. The proof can be generalized to any view easily.
\end{proof}

\subsection{Linear Computation Complexity}

DMAC is able to accomplish MVC in the linear time complexity $O(n)$. To avoid excessive symbol definition and improve readability, we only analyze the influence of sample size $n$ and anchor number $m$.

In each forward propagation, the computation complexity of embedding learning and fusion is $O(n)$. Then, the initial anchor selection needs $O(nm)$ via $k$-means, and the generator requires $O(m)$ to output the perturbation matrix. Finally, considering that $\mathbf{D}_\mathbf{S^{\mathnormal{(a)}}}^{-1}$ is a diagonal matrix, the consumption of anchor graph convolution is $O(nm)$.

In each back propagation, the anchor learning loss needs $O(nm)$ to calculate the entropies of all rows in $\mathbf{Q}$. The structure preservation loss can be written as the trace form
\begin{equation}
	\label{eq:sp_loss_speed}
	\begin{array}{c}
		\mathrm{Tr}\left(\mathbf{Z}^{\mathrm{T}}\mathbf{D}_\mathbf{G}^{(a)}\mathbf{Z} - \mathbf{Z}^{\mathrm{T}} \mathbf{S}^{(a)} {\mathbf{D}^{{-1}}_{\mathbf{S}^{(a)}}} \left(\mathbf{S}^{(a)}\right)^{\mathrm{T}}\mathbf{Z}  \right),
	\end{array}
\end{equation}
where $\mathbf{D}_\mathbf{G}^{(a)} \in \mathbb{R}^{n\times n}$ is the degree matrix of $\mathbf{G}^{(a)} $ shown in Eq. (\ref{eq:full_graph}). Since $\mathbf{G}^{(a)}$ is a doubly stochastic matrix, $\mathbf{D}_\mathbf{G}^{(a)} \in \mathbb{R}^{n\times n}$ is an identity matrix. The calculation consumption of structure preservation loss is also $O(nm)$ with Eq. (\ref{eq:sp_loss_speed}). Finally, the consistency maximization loss needs $O(m)$. 

In conclusion, the time complexity of DMAC is $O(nm)$. Normally, the quantity of anchors is much smaller than the sample size (i.e., $m \ll n$), so the average complexity of each iteration can be seen as $O(n)$.

\section{Experiments}

In this section, the proposed DMAC is compared with advanced competitors. The ablation analysis is also conducted.

\subsection{Real-World Datasets}

\begin{table}[t]
	\tabcolsep=0.09cm
	\centering
	\small
	\begin{tabular}{ccccl}
		\toprule
		\textbf{Dataset} & \textbf{Samples} & \textbf{Views} & \textbf{Classes} & \textbf{Dimensions}\\ 
		\midrule
		\textbf{Yale} & 165 & 3 & 15 & 4096, 3304, 6750 \\
		\textbf{PIE} & 680 & 3 & 68 & 484, 256, 279 \\
		\textbf{BBC} & 685 & 4 & 5 & 4659, 4633, 4665, 4684 \\
		\textbf{NUS} & 2400 & 6 & 10 & 64, 144, 73, 128, 225, 500 \\
		\textbf{CCV} & 6773 & 3 & 20 & 4000, 5000, 5000 \\
		\textbf{ALOI} & 10800 & 4 & 100 & 77, 13, 64, 125 \\
		\bottomrule
	\end{tabular}
	\caption{Descriptions of real-world datasets.}
	\label{tab:datasets}
\end{table}

\begin{table*}[t]
	\tabcolsep=0.18cm
	\centering
	\small
	\begin{tabular}{ccc|cc|cc|cc|cc|cc|cc}
		\toprule
		\makecell[c]{\multirow{2.5}{*}{Method}} & \multicolumn{2}{c}{Yale} & \multicolumn{2}{c}{PIE} & \multicolumn{2}{c}{BBC} & \multicolumn{2}{c}{NUS} & \multicolumn{2}{c}{CCV} & \multicolumn{2}{c}{ALOI} & \multicolumn{2}{c}{Avg}  \\
		\cmidrule{2-15} 
		\makecell[c]{\multirow{2.5}{*}{}} & \makecell[c]{ACC} & \makecell[c]{NMI} & \makecell[c]{ACC} & \makecell[c]{NMI} & \makecell[c]{ACC} & \makecell[c]{NMI} & \makecell[c]{ACC} & \makecell[c]{NMI} & \makecell[c]{ACC} & \makecell[c]{NMI} & \makecell[c]{ACC} & \makecell[c]{NMI} & \makecell[c]{ACC} & \makecell[c]{NMI} \\
		\midrule
		\makecell[c]{GMC} & \makecell[c]{69.70} & \makecell[c]{70.06} & \makecell[c]{21.18} & \makecell[c]{44.24} & \makecell[c]{69.05} & \makecell[c]{47.87} & \makecell[c]{18.24} & \makecell[c]{9.96} & \makecell[c]{10.66} & \makecell[c]{0.43} & \makecell[c]{57.05} & \makecell[c]{73.50} & \makecell[c]{40.98} & \makecell[c]{41.01} \\
		\makecell[c]{MSGL} & \makecell[c]{40.61} & \makecell[c]{47.32} & \makecell[c]{15.74} & \makecell[c]{46.08} & \makecell[c]{46.28} & \makecell[c]{23.15} & \makecell[c]{15.25} & \makecell[c]{5.27} & \makecell[c]{12.42} & \makecell[c]{7.11} & \makecell[c]{15.81} & \makecell[c]{39.66} & \makecell[c]{24.35} & \makecell[c]{28.10}\\
		\makecell[c]{LMVSC} & \makecell[c]{57.58} & \makecell[c]{58.10} & \makecell[c]{\underline{36.32}} & \makecell[c]{63.33} & \makecell[c]{66.42} & \makecell[c]{53.92} & \makecell[c]{20.67} & \makecell[c]{8.58} & \makecell[c]{18.29} & \makecell[c]{14.09} & \makecell[c]{\underline{58.51}} & \makecell[c]{\textbf{76.37}} & \makecell[c]{42.97} & \makecell[c]{45.73}\\
		\makecell[c]{UDBGL} & \makecell[c]{53.33} & \makecell[c]{58.76} & \makecell[c]{24.26} & \makecell[c]{52.72} & \makecell[c]{72.85} & \makecell[c]{50.94} & \makecell[c]{24.08} & \makecell[c]{13.11} & \makecell[c]{25.57} & \makecell[c]{20.83} & \makecell[c]{52.44} & \makecell[c]{61.02} & \makecell[c]{42.09} & \makecell[c]{42.90}\\
		\midrule
		\makecell[c]{CMGEC} & \makecell[c]{36.36}  & \makecell[c]{42.60} & \makecell[c]{14.77} & \makecell[c]{45.33} & \makecell[c]{\underline{87.37}} & \makecell[c]{\underline{71.44}} & \makecell[c]{24.87} & \makecell[c]{10.83} & \makecell[c]{22.21} & \makecell[c]{23.67} & \makecell[c]{56.42} & \makecell[c]{72.89} & \makecell[c]{40.33} & \makecell[c]{44.46}\\
		\makecell[c]{DealMVC} & \makecell[c]{\underline{75.18}} & \makecell[c]{\underline{76.81}} & \makecell[c]{23.82} & \makecell[c]{52.14} & \makecell[c]{64.75} & \makecell[c]{41.20} & \makecell[c]{20.04} & \makecell[c]{9.49} & \makecell[c]{13.95} & \makecell[c]{6.87} & \makecell[c]{17.50} & \makecell[c]{44.50} & \makecell[c]{35.87} & \makecell[c]{38.50}\\
		\makecell[c]{GCFAggMVC} & \makecell[c]{66.06} & \makecell[c]{66.51} & \makecell[c]{27.94} & \makecell[c]{59.15} & \makecell[c]{63.65} & \makecell[c]{48.87} & \makecell[c]{23.42} & \makecell[c]{10.69} & \makecell[c]{\underline{35.43}} & \makecell[c]{\underline{32.92}} & \makecell[c]{54.52} & \makecell[c]{72.21} & \makecell[c]{\underline{45.17}} & \makecell[c]{48.39}\\
		\makecell[c]{DFP-GNN} & \makecell[c]{56.36} & \makecell[c]{63.39} & \makecell[c]{24.26} & \makecell[c]{56.88} & \makecell[c]{75.09} & \makecell[c]{58.73} & \makecell[c]{\textbf{29.42}} & \makecell[c]{16.12} & \makecell[c]{21.33} & \makecell[c]{19.36} & \makecell[c]{49.15} & \makecell[c]{66.12} & \makecell[c]{42.60} & \makecell[c]{46.77}\\
		\makecell[c]{SURER} & \makecell[c]{61.82} & \makecell[c]{67.68} & \makecell[c]{30.29} & \makecell[c]{\underline{64.16}} & \makecell[c]{79.85} & \makecell[c]{64.24} & \makecell[c]{27.33} & \makecell[c]{\underline{16.17}} & \makecell[c]{24.91} & \makecell[c]{26.86} & \makecell[c]{43.94} & \makecell[c]{64.03} & \makecell[c]{44.69} & \makecell[c]{\underline{50.52}}\\
		\midrule
		\makecell[c]{DMAC} & \makecell[c]{\textbf{78.18}} & \makecell[c]{\textbf{78.06}} & \makecell[c]{\textbf{43.24}} & \makecell[c]{\textbf{68.16}} & \makecell[c]{\textbf{88.61}} & \makecell[c]{\textbf{74.49}} & \makecell[c]{\underline{29.29}} & \makecell[c]{\textbf{16.20}} & \makecell[c]{\textbf{36.18}} & \makecell[c]{\textbf{33.17}} & \makecell[c]{\textbf{60.35}} & \makecell[c]{\underline{74.29}} & \makecell[c]{\textbf{55.98}} & \makecell[c]{\textbf{57.40}}\\
		\bottomrule
	\end{tabular}
	\caption{Clustering performance of ten methods on six datasets. Bold and underlined values mean the optimal and sub-optimal results respectively. The column termed avg displays the average ACC and NMI of each method.}
	\label{tab:comparsion_performance}
\end{table*}

%\footnote{http://cvc.yale.edu/projects/yalefaces/yalefaces.html}
%\footnote{https://www.cs.cmu.edu/afs/cs/project/PIE/MultiPie/Multi-Pie/Home.html}
%\footnote{https://elki-project.github.io/datasets/multi$\_$view}
%\footnote{http://mlg.ucd.ie/datasets/bbc.html}
%\footnote{https://www.comp.nus.edu.sg/\textasciitilde nlp/corpora.html}
%\footnote{https://www.ee.columbia.edu/ln/dvmm/CCV/}

Six public real-world datasets that are widely used in clustering study are collected as benchmarks, including image-type Yale \cite{Yale}, PIE \cite{PIE} and ALOI \cite{ALOI}, text-type BBC \cite{BBC} and NUS \cite{NUSWIDE}, and video-type CCV \cite{CCV}. Each sample is preprocessed with the $\ell_{2}$ norm normalization. Table \ref{tab:datasets} displays the basic information of each dataset. 

\begin{figure}[t]
	\center
	\begin{minipage}[t]{0.43\textwidth}
		\centering
		\includegraphics[width=1\textwidth]{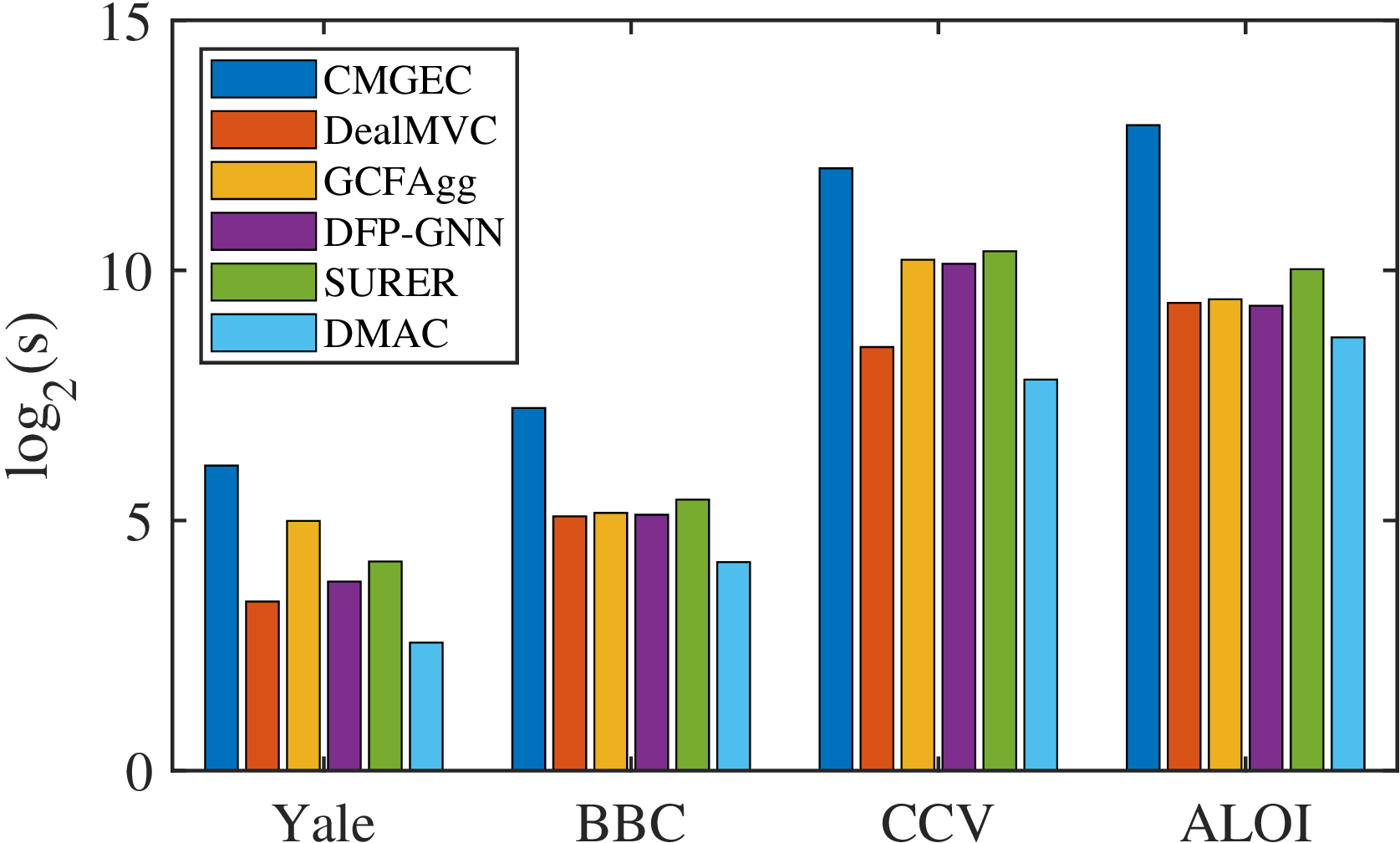}
	\end{minipage}
	\caption{Runtime (s) of deep models on four datasets. Note that all records are converted by logarithmic base 2.}
	\label{fig:speed}
\end{figure} 

\subsection{Evaluation Metrics}

Two widespread metrics are adopted to quantify the clustering result, including Accuracy (ACC) and Normalized Mutual Information (NMI). Both ACC and NMI are positively correlated with the clustering performance. The mathematical expression can be found in \cite{metric}. 

\subsection{Comparison with Competitors}

\subsubsection{Competitors.} Nine state-of-the-art methods are selected as competitors, including four shallow algorithms GMC \cite{GMC}, MSGL \cite{MSGL}, LMVSC \cite{LMVSC} and UDBGL \cite{UDBGL}, and five deep models CMGEC \cite{CMGEC}, DealMVC \cite{DealMVC}, GCFAggMVC \cite{GCFAgg}, DFP-GNN \cite{DFP-GNN} and SURER \cite{SURER}. Among them, MSGL, LMVSC and UDBGL are anchor-based MVC, and all deep models incorporate the graph structure information.

\subsubsection{Setups.} The grid search is used to explore the optimal parameter setup for each algorithm.
The parameter grid of competitors are set as the recommendations in the original article. For example, the parameter $\alpha$ of MSGL is selected from \{0.001, 0.01, 0.1, 1, 10, 50\}. For the proposed DMAC, the number of anchors $m$ is set automatically according to \cite{KMM}, i.e., the lower bound of $\sqrt{n \times c}$. The grid for both $\alpha$ and $\beta$ is \{$10^{-3}$, $10^{-2}$, $10^{-1}$, $1$, $10^{1}$ ,$10^{2}$, $10^{3}$\}. The maximal iterations are 100.

The traditional methods are executed on Matlab 2019a with an Intel i9-12900HX CPU. All deep models are implemented via PyTorch, and trained with a NVIDIA RTX-3090 GPU. Each algorithm is repeated 10 times for objectivity. 

\subsubsection{Performance Comparison.}

\begin{figure}[t]
	\center
	\begin{minipage}[t]{0.15\textwidth}
		\centering
		\includegraphics[width=1.03\textwidth]{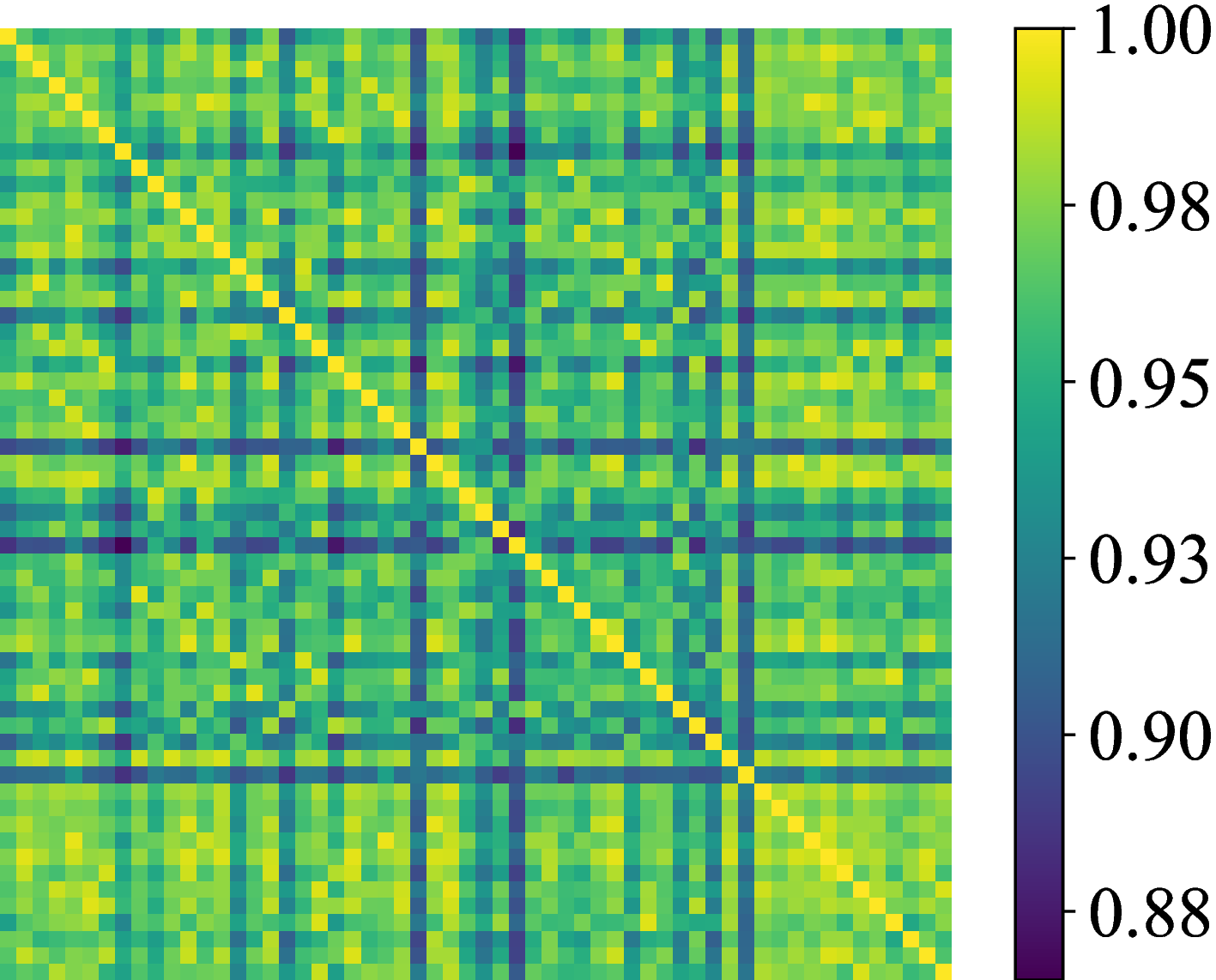}
		\centerline{wo/PD}
	\end{minipage}
	\hspace{0.1pt}
	\begin{minipage}[t]{0.15\textwidth}
		\centering
		\includegraphics[width=1\textwidth]{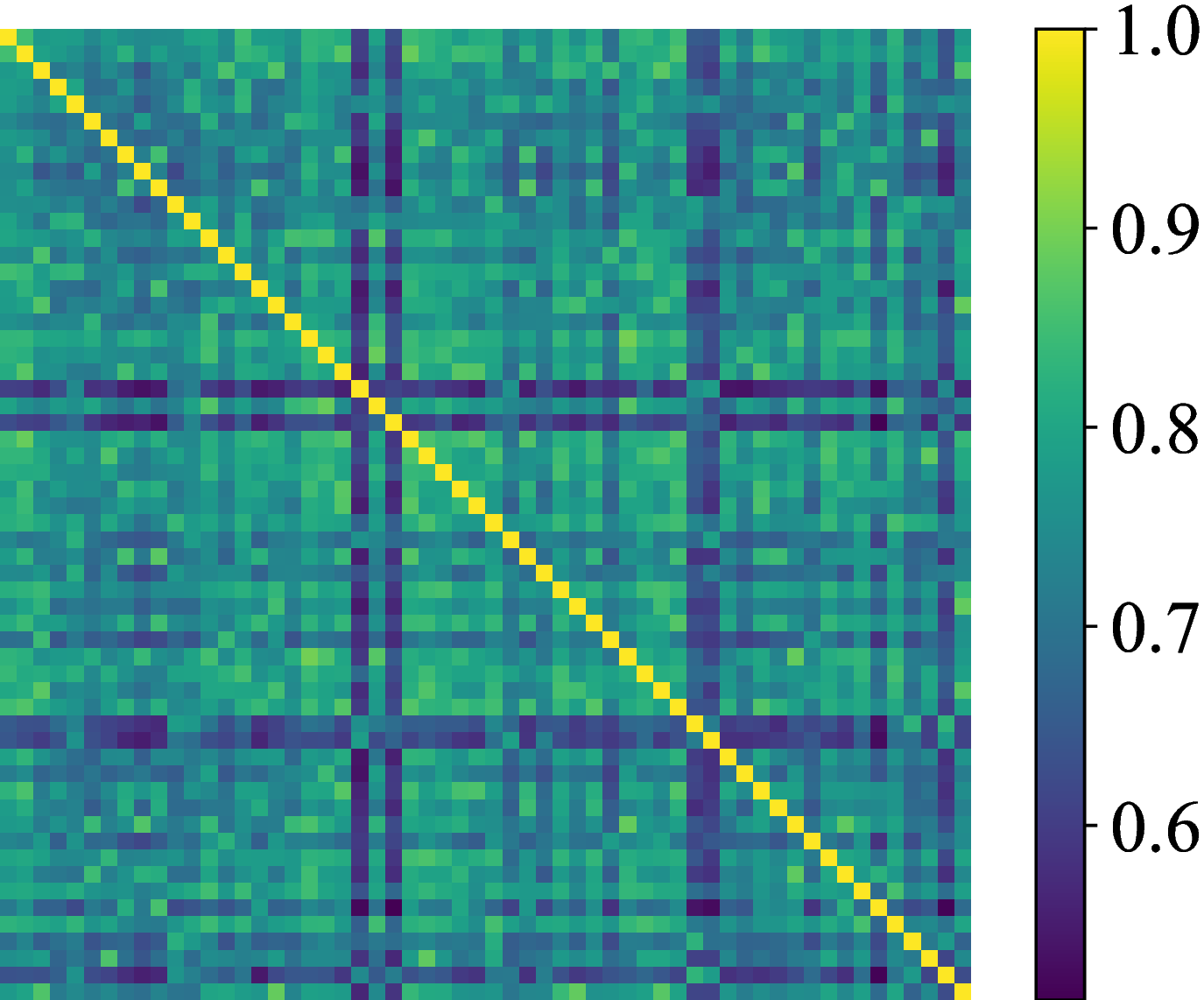}
		\centerline{wo/CM}
	\end{minipage}
	\begin{minipage}[t]{0.15\textwidth}
		\centering
		\includegraphics[width=1\textwidth]{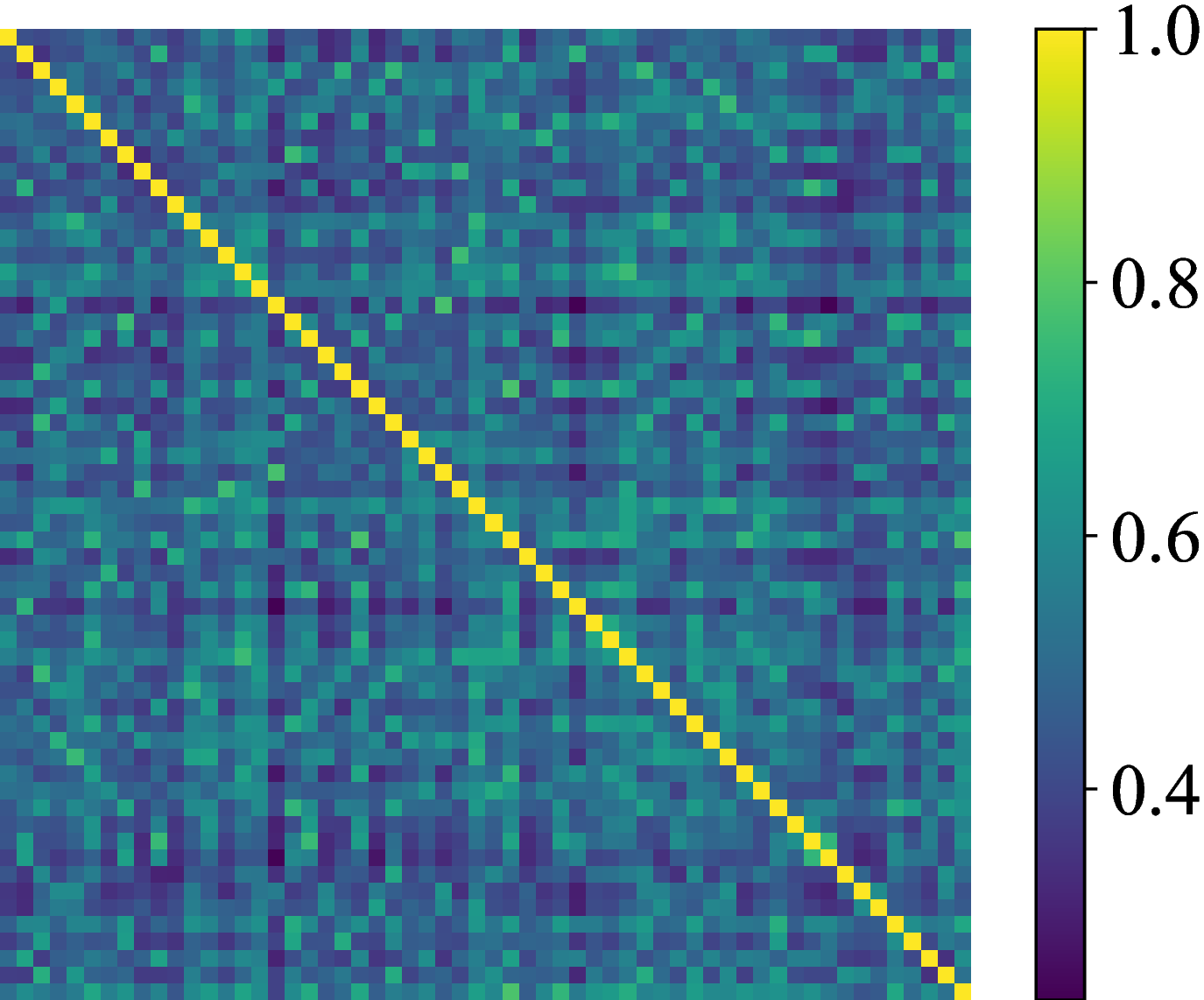}
		\centerline{DMAC}
	\end{minipage}
	\caption{Anchor similarity matrix $\mathbf{U}\mathbf{U}^{\mathrm{T}}$ on BBC. }
	\label{fig:simAnchor}
\end{figure}

Table \ref{tab:comparsion_performance} records the clustering performance of all algorithms. For ease of comparison, we also calculate the average ACC and NMI of each algorithm on all datasets. In general, DMAC presents the best clustering ability. The success of DMAC proves the feasibility of applying anchor graph learning to deep MVC. DMAC learns high-quality anchors with the proposed perturbation-driven anchor learning scheme, and then mines the multi-view anchor clustering consistency via anchor graph convolution and mutual information maximization to further accelerate clustering-oriented anchors, so as to accurately reveal the structural graph for clustering improvement. According to the experimental results, we also summarize the following viewpoints. Firstly, compared with the traditional shallow methods, the deep models achieve better clustering scores, which indicates the enormous potential of neural networks on improving MVC. Secondly, the performance of GCFAggMVC and SURER are more prominent than other deep methods. SURER and GCFAggMVC leverage data similarity graphs to guide heterogeneous graph embedding learning and feature aggregation respectively, which reflects the positive effects of structural information on the two key steps of deep MVC, namely, view-specific representation learning and multi-view representation fusion. 

\subsubsection{Efficiency Comparison.} 

%\begin{table}[t]
%	% \setlength{\belowcaptionskip}{-0.3cm}
%	\centering
%	\small
%	\begin{tabular}{ccccc}
	%		\toprule
	%		Method & Yale & BBC & CCV & ALOI\\ 
	%		\midrule
	%		CMGEC & 68.46 & 151.97 & 4212.79 & 12903.75 \\
	%		DealMVC & 10.42 & 33.93 & 353.81 & 652.72 \\
	%		GCFAggMVC & 31.82 & 35.56 & 1186.54 & 1519.27 \\
	%		DFP-GNN & 13.73 & 34.67 & 1123.24 & 865.42 \\
	%		SURER & 18.13 & 42.78 & 1336.64 & 1040.37 \\
	%		\midrule
	%		DMAC & \textbf{5.92} & \textbf{18.02} & \textbf{225.67} & \textbf{404.90}\\
	%		\bottomrule
	%	\end{tabular}
%	\caption{Runtime (s) of deep models. All records are converted by logarithmic base 2.}
%	\label{tab:speed}
%\end{table}

Fig. \ref{fig:speed} displays the clustering efficiency of deep models. It is observed that DMAC has the shortest runtime. Compared with advanced deep MVC models that incorporate graph structure learning, DMAC avoids inefficient full sample graph learning and graph convolution. The new anchor learning mechanism and anchor graph convolution network have linear time complexity theoretically, so as to speed up the training process. 
%In addition, Since the initial dimension of CCV is much larger than that of ALOI, some models (e.g., DFP-GNN and SURER) are more efficient when clustering ALOI than CCV.

\subsection{Ablation Study and Visualization}

\begin{table}[t]
	\centering
	\small
	\begin{tabular}{c|c|ccc}
		\toprule
		Dataset & Metric & wo/PD & wo/CM & DMAC \\ 
		\midrule
		\multirow{2}{*}{Yale} & ACC & 66.02 & 72.73 & \textbf{78.18} \\ 
		\multirow{2}{*}{} & NMI & 66.90 & 75.15 & \textbf{78.06} \\ 
		\midrule
		\multirow{2}{*}{PIE} & ACC & 34.12 & 33.82 & \textbf{43.24} \\ 
		\multirow{2}{*}{} & NMI & 62.94 & 62.97 & \textbf{68.16} \\ 
		\midrule
		\multirow{2}{*}{BBC} & ACC & 87.15 & 88.47 & \textbf{88.61} \\ 
		\multirow{2}{*}{} & NMI & 70.68 & 73.99 & \textbf{74.49} \\ 
		\midrule
		\multirow{2}{*}{NUS} & ACC & 24.53 & 27.67 & \textbf{29.29} \\ 
		\multirow{2}{*}{} & NMI & 14.29 & 15.35 & \textbf{16.20} \\ 
		\midrule
		\multirow{2}{*}{CCV} & ACC & 32.28 & 34.26 & \textbf{36.18}   \\ 
		\multirow{2}{*}{} & NMI & 29.96 & 31.72 & \textbf{33.17}  \\ 
		\midrule
		\multirow{2}{*}{ALOI} & ACC & 44.45 & 54.79 & \textbf{60.35}  \\ 
		\multirow{2}{*}{} & NMI & 65.86 & 70.10 & \textbf{74.29}  \\ 
		\bottomrule
	\end{tabular}
	\caption{Ablation results of main modules in DMAC. Bold values emphasize the optimal results.}
	\label{tab:ablation_performance}
\end{table}

\begin{figure}[t]
	\center
	\begin{minipage}[t]{0.14\textwidth}
		\centering
		\includegraphics[width=1\textwidth]{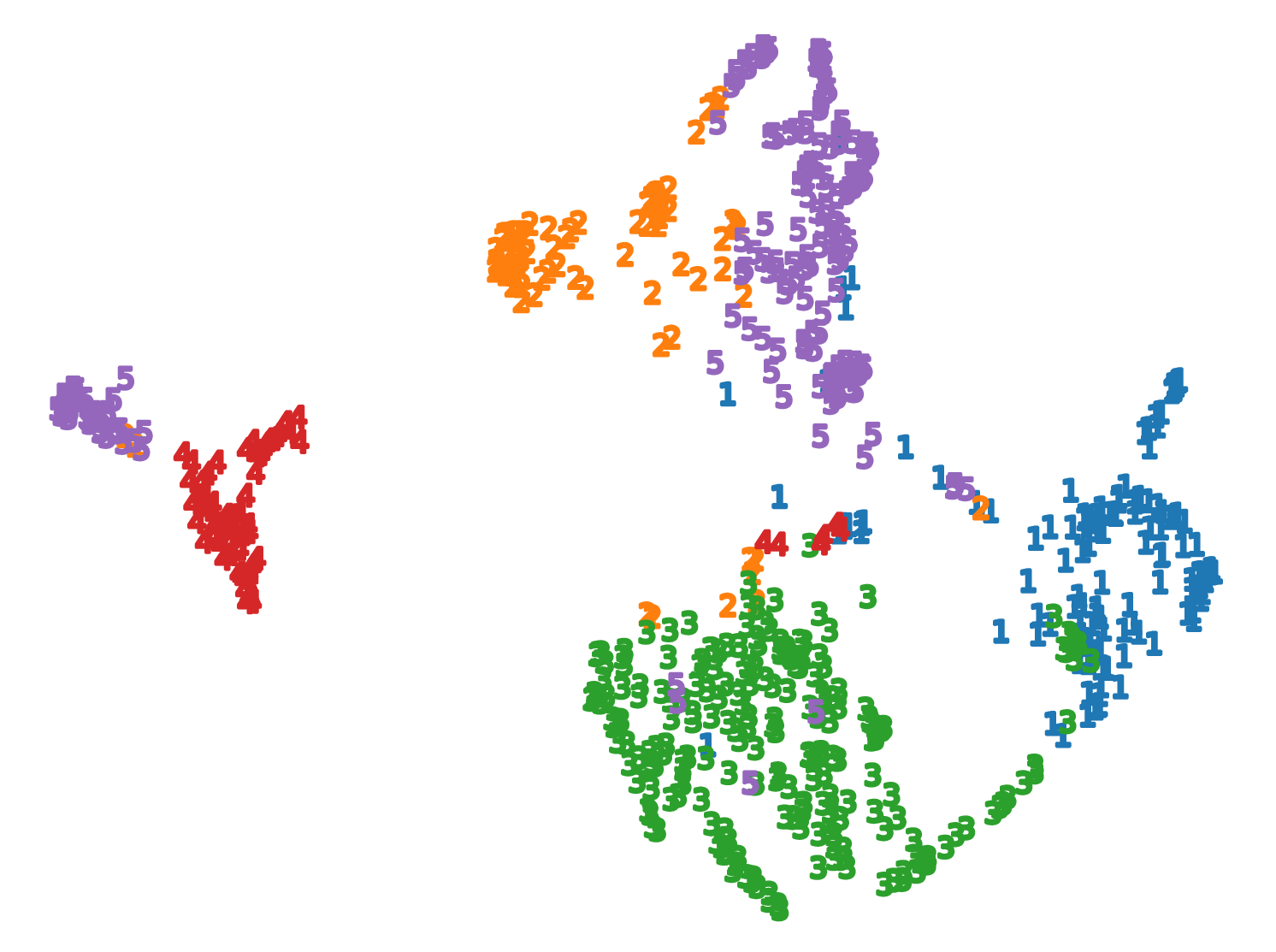}
		\centerline{wo/PD}
	\end{minipage}
	\hspace{3pt}
	\begin{minipage}[t]{0.14\textwidth}
		\centering
		\includegraphics[width=1\textwidth]{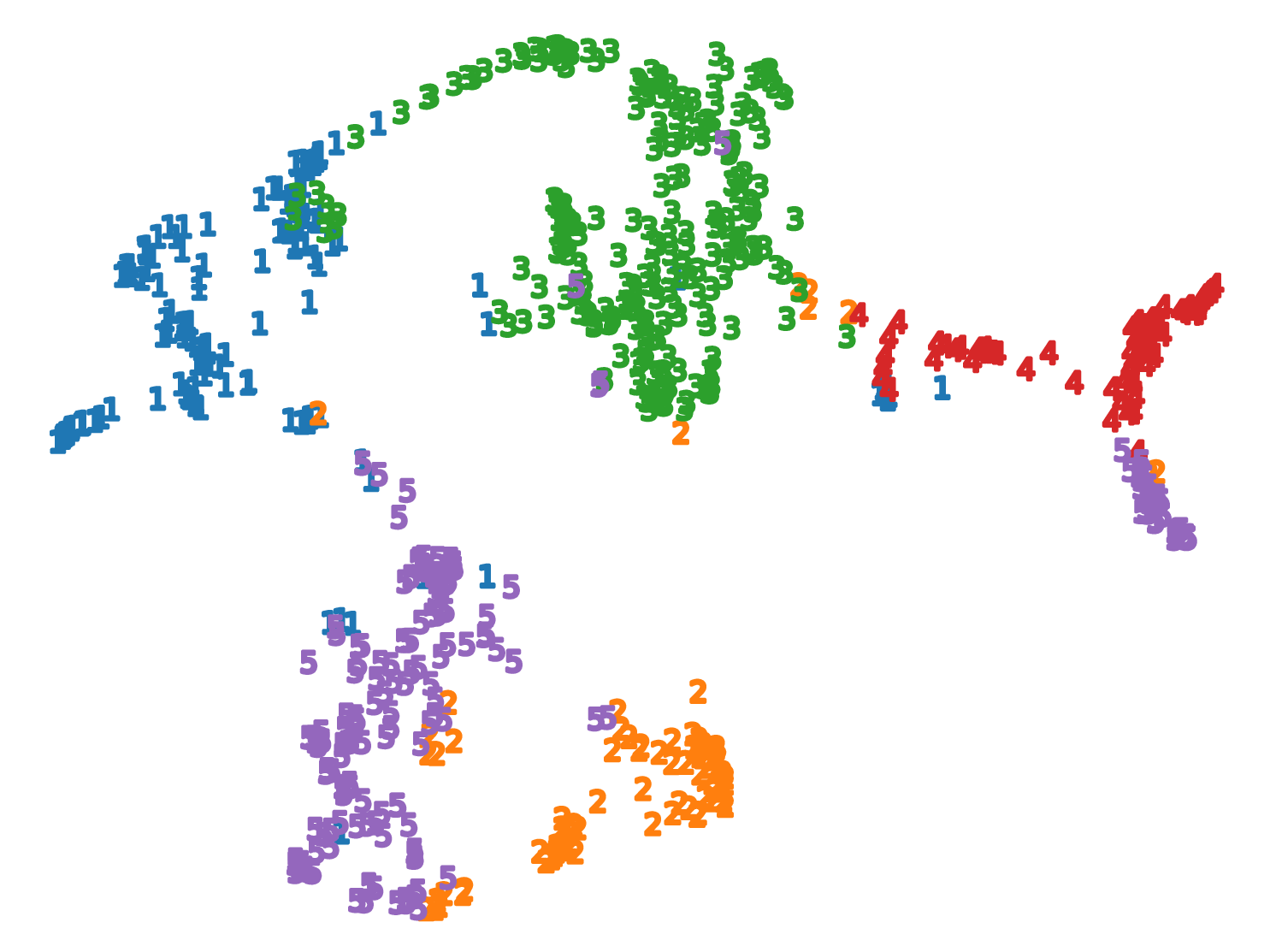}
		\centerline{wo/CM}
	\end{minipage}
	\hspace{3pt}
	\begin{minipage}[t]{0.14\textwidth}
		\centering
		\includegraphics[width=1\textwidth]{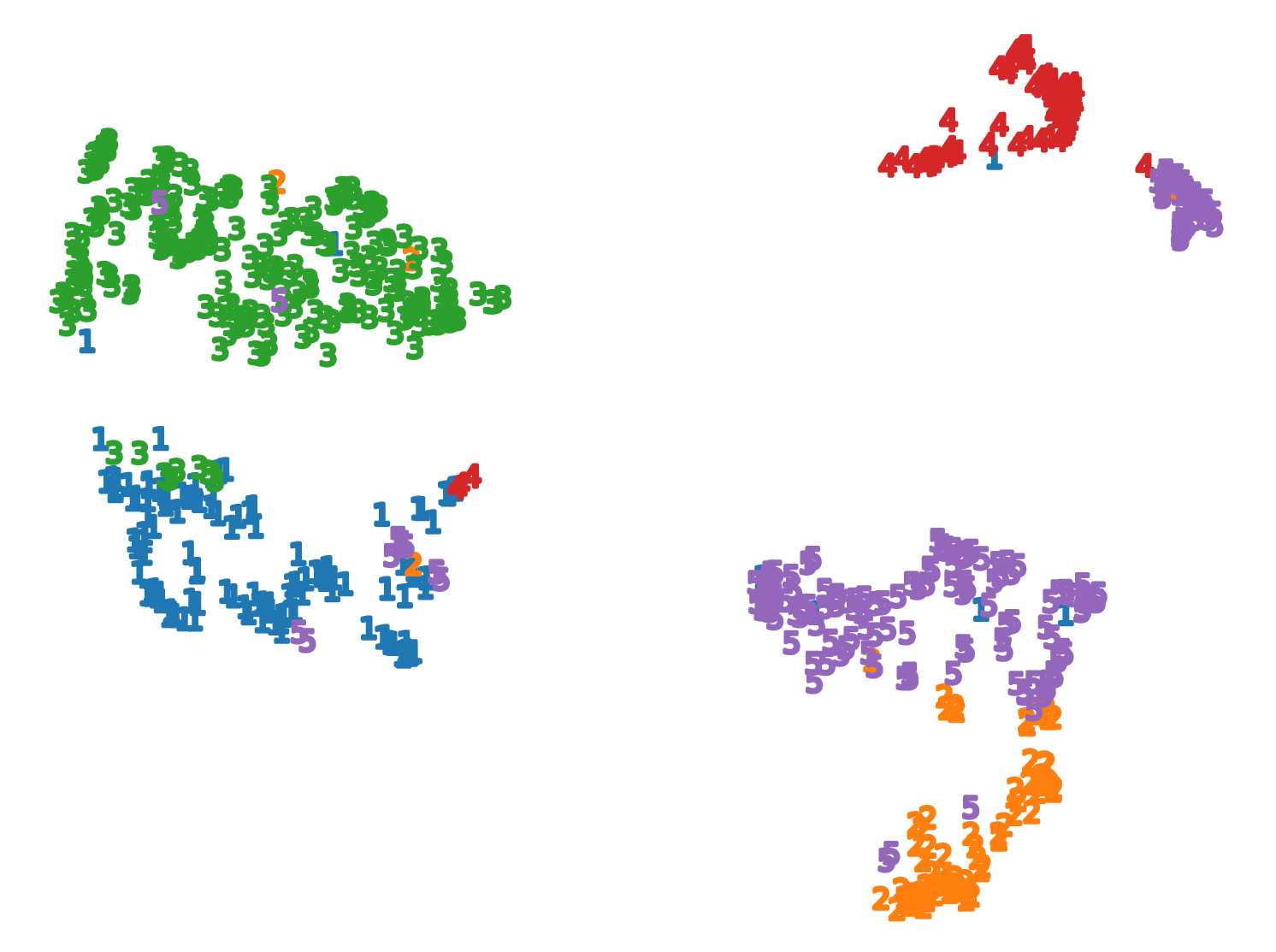}
		\centerline{DMAC}
	\end{minipage}
	\caption{Visualization of fusion embedding $\mathbf{Z}$ on BBC. Each point is drawn as its actual label value.}
	\label{fig:simFuse}
\end{figure}

In the ablation experiment, we design two variants based on the complete DMAC. Concretely, wo/PD removes the perturbation generation network and anchor learning loss, and wo/CM suspends the consistency maximization loss. Table \ref{tab:ablation_performance} shows the ablation comparison. DMAC still maintains the best clustering performance, which proves the positive role of the proposed modules.

To intuitively display the effects of new modules, we visualize the similarity matrix of anchors (i.e., $\mathbf{U}\mathbf{U}^{\mathrm{T}}$) in Fig. \ref{fig:simAnchor}. It is exhibited that the anchor similarity matrix learned by DMAC has the most distinct diagonal, which means the anchors are relatively dispersive to adequately represent the sample clusters. The significant degradation of wo/PD compared to DMAC further proves the advantages of learnable anchors. The disparity between wo/CM and DMAC indicates that mining cross-view anchor clustering information is beneficial to guide clustering-oriented anchor learning for performance improvement. 

In addition, we utilize UMAP \cite{umap} to visualize the learned fusion embedding $\mathbf{Z}$ on BBC. As shown in Fig. \ref{fig:simFuse}, the result derived by DMAC has a more obvious inter-class partition, which means a small inter-class similarity. The ablation comparison again indicates that, the new modules are conducive to learn high-quality and clustering-friendly anchors for accurate structure learning, so as to improve the discriminative ability of fusion embedding via the structure preservation loss. The advantage of DMAC over wo/PD reflects the practicability of Theorem \ref{the:trivial}, that is, the proposed anchor learning loss can suppress representation collapse.

\subsection{Parameter Sensitivity} 

\begin{figure}[t]
	\center
	\begin{minipage}[t]{0.22\textwidth}
		\centering
		\includegraphics[width=1\textwidth]{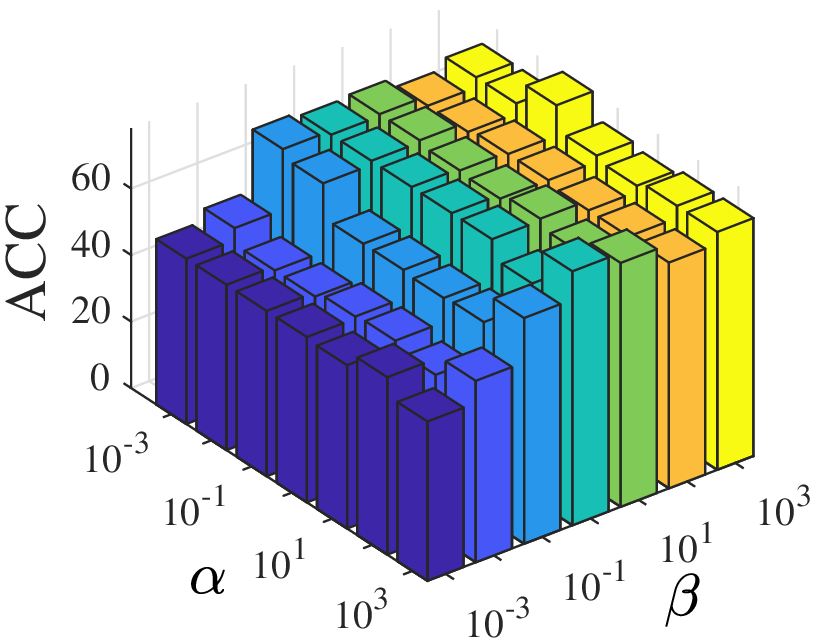}
		\centerline{Yale}
	\end{minipage}
	\hspace{3pt}
	\begin{minipage}[t]{0.22\textwidth}
		\centering
		\includegraphics[width=1\textwidth]{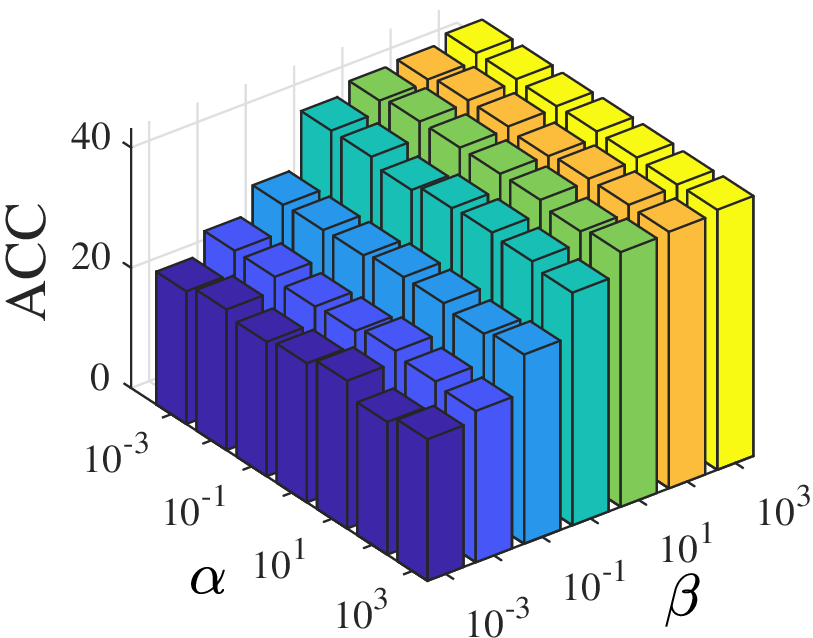}
		\centerline{PIE}
	\end{minipage}
	\center
	\begin{minipage}[t]{0.22\textwidth}
		\centering
		\includegraphics[width=1\textwidth]{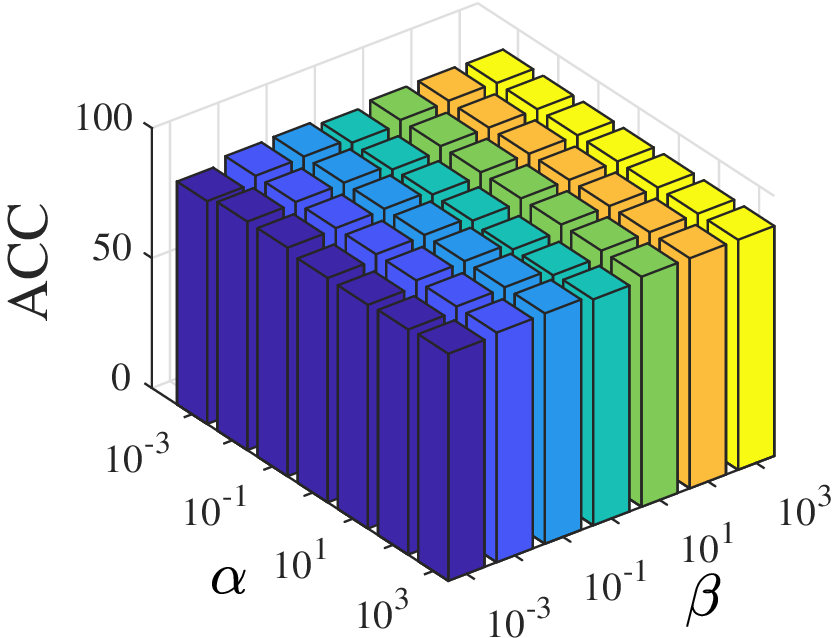}
		\centerline{BBC}
	\end{minipage}
	\hspace{3pt}
	\begin{minipage}[t]{0.22\textwidth}
		\centering
		\includegraphics[width=1\textwidth]{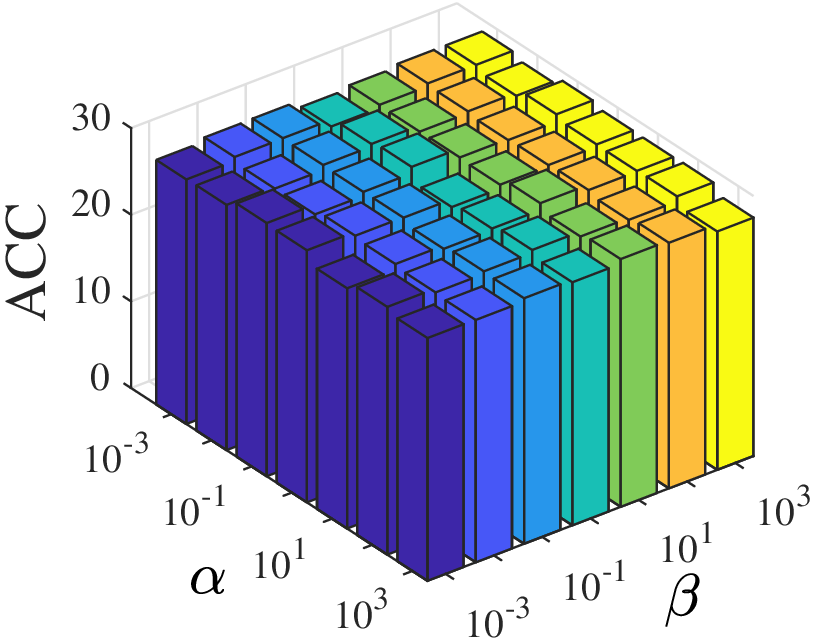}
		\centerline{NUS}
	\end{minipage}
	\caption{ACC of DMAC with different parameters $\alpha$ and $\beta$.}
	\label{fig:par_sensitivity}
\end{figure}

Finally, the sensitivity of DMAC to the trade-off parameters $\alpha$ and $\beta$ is explored. Fig. \ref{fig:par_sensitivity} exhibits ACC of DMAC under the predefined parameter grid. It can be seen that the influence of $\beta$ is more significant than that of $\alpha$, because $\beta$ is directly related to the fusion representation $\mathbf{Z}$ that derives the final result. The preliminary observation suggests that the combination of large $\alpha$ and $\beta$ is more likely to facilitate a good clustering result. Overall, the performance fluctuation is relatively smooth within an appropriate range. 

\section{Conclusion}

In this paper, we propose an anchor-based deep multi-view clustering model termed DMAC. 
Different from traditional manual anchor selection ways, DMAC introduces a perturbation-driven anchor learning mechanism to make the anchors learnable. 
Specifically, inspired by the positive-incentive noise theory, a noise generation network is established to produce the perturbation adaptively, which is injected into anchors under the guidance of anchor learning loss. 
Besides, the anchor graph convolution module is designed to extract the cluster structure of anchors within each view, and then the multi-view anchor clustering consistency can be perceived with mutual information maximization.
In this manner, DMAC is able to optimize the anchors during the training procedure, and pursue a desired anchor distribution for clustering.
Theoretical analysis shows that DMAC has a linear time complexity $O(n)$. Experiments report the superior performance and efficiency of DMAC.

\section{Acknowledgments}

This work was supported by the National Key Research and Development Program of China (Grant No: 2022ZD0160803), and the National Natural Science Foundation of China (Grant No: 61871470).

\bibliography{aaai25}

\end{document}